\documentclass[12pt,twoside]{article}
\usepackage{latexsym,amsmath,amssymb,amsthm}
\def\argmax{\operatornamewithlimits{argmax}}
\def\argmin{\operatornamewithlimits{argmin}}

\newtheorem{theorem}{Theorem}
\newtheorem{definition}[theorem]{Definition}
\newtheorem{example}[theorem]{Example}

\newtheorem{remark}[theorem]{Remark}

\newenvironment{keywords}{\centerline{\bf\small
Keywords}\begin{quote}\small}{\par\end{quote}\vskip 1ex}

\begin{document}

\title{
\vskip 2mm\bf\Large\hrule height5pt \vskip 4mm
Principles of Solomonoff Induction and AIXI
\vskip 4mm \hrule height2pt}

\author{{\bf Peter Sunehag}$^1$ and {\bf Marcus Hutter}$^{1,2}$\\[1mm]
\texttt{\normalsize \{Peter.Sunehag,Marcus.Hutter\}@anu.edu.au}\\[1mm]
\normalsize $^1$Research School of Computer Science, Australian National University\\[-0.5ex]
\normalsize Canberra, ACT, 0200, Australia\\[-0.5ex]
\normalsize $^2$Department of Computer Science, ETH Z{\"u}rich, Switzerland\\
}

\date{November 2011}

\maketitle

\begin{abstract}
We identify principles characterizing Solomonoff Induction by
demands on an agent's external behaviour. Key concepts are
rationality, computability, indifference and time consistency.
Furthermore, we discuss extensions to the full AI case to
derive AIXI.
\def\contentsname{\centering\normalsize Contents}\setcounter{tocdepth}{1}
{\parskip=-2.7ex\tableofcontents}
\end{abstract}


\newpage
\section{Introduction}

Ray Solomonoff \cite{Sol60} introduced a universal sequence
prediction method that in \cite{Solomonoff96,Hutter07,Sam} is
argued to solve the general induction problem. \cite{Hutter04}
extended Solomonoff induction to the full AI (general
reinforcement learning) setting where an agent is taking a
sequence of actions that may affect the unknown environment to
achieve as large amount of reward as possible. The resulting
agent was named AIXI. Here we take a closer look at what
principles underlie Solomonoff induction and the AIXI agent. We
are going to derive Solomonoff induction from four general
principles and discuss how AIXI follows from extended versions
of the same.

Our setting consists of a reference universal Turing machine
(UTM), a binary sequence (produced by an environment program
(not revealed) on the reference machine) fed incrementaly to
the agent and a loss function (or reward structure). We give
the agent in question the task of choosing a program for the
reference machine so as to minimize the loss. The loss is in
general defined to be a function from a pair of programs, an
environment program and an agent program, to real numbers. The
loss function can be such that it is only the prediction (for a
certain number of bits) produced by the program that matters or
it can care about exactly which program was presented. A loss
function of the latter kind leads to the agent performing the
task of prediction, which is what Solomonoff induction is
primarily concerned with while the latter can be viewed as
identifying an explanatory hypothesis, which is more closely
related to the minimum message length principle
\cite{Wallace68, Wallace99,wallace05} or the minimum
description length principle \cite{rissanen78,
grunwald,rissanen10}. Solomonoff induction is using a mixture
of hypothesis to achieve the best possible prediction. Note
that the fact that we pick one program does not rule out that
the choice is internally based on a mixture. In the case when
the loss only cares about the prediction, the program is only a
representation of that prediction and not really a hypothesis.

The principles are designed to avoid stating what the internal
workings of the agent should be and instead derive those as a
consequence of the demands on the behaviour. Thus we demand
rationality instead of stating explicitly that the agent should
have probabilistic beliefs and we demand time consistency
instead of explicitly stating probabilistic conditioning.  The
computability principle is avoiding saying that the agent
should have a hypothesis class that consists of all computable
environments by instead demanding that it deliver a computation
procedure (a program for our reference machine) that produces
its prediction for the next few bits.The indifference principle
states what the initial preferences of the agent must be, i.e.\
a demand for how the initial decision should be taken. The
choice is based on symmetry with respect to a chosen
representation scheme for sequences, e.g.\ programs on a
reference machine. In other words we do not allow the agent to
be biased in a certain sense that depends on our reference
machine. Informally we state the principles as follows:

\begin{enumerate}
\item {\bf Computability}: If we are going to guess the
    future of a sequence, we should choose a computation
    procedure (a program for the reference machine) that
    produces the predicted bits
\item {\bf Rationality}: We should choose our predicted
    sequence such that the dependence on the priorities
    (formalized by a reward (or loss) structure) is
    consistent.
\item {\bf Indifference}: The initial choice between
    programs only depends on their length and the
    priorities (again formalized by reward (or loss))
\item {\bf Time Consistency}: The choice of program does
    not change by a new observation if the program's output
    is consistent with the oberservation and the reward
    structure is still the same and concerned with the same
    bits
\end{enumerate}

Our reasoning leading from external behavioural principles to a
completely defined internal procedure can be summarized as
follows; The rationality principle tells us that we need to
have probabilistic beliefs over some set of alternatives; The
computability principle tells us what the alternatives are,
namely programs; The indifference principle leads to a  choice
of the original beliefs; The time-consistency principle leads
to a simple procedure for updating the beliefs that the second
principle tells us must exist, namely conditioning.  In total
it leads to Solomonoff Induction.

We can not remove any of the principles without losing the
complete specification of a procedure. The first property is
part of the set up of what we ask the agent to do. Without the
second we lose the restriction that we take decisions based on
maximum expected utility with respect to probabilistic beliefs
and one could then have an agent that always chose the same
program (e.g.\ a very short one).  Without the third principle
we could have any apriori beliefs and without the fourth the
agent could after a while change its mind regarding what
beliefs it started with.

\subsection{Setup}
We are considering a setting where we give an agent a task that
is defined by a reference machine (a UTM), a reward structure
(or loss function if we negate) and a binary sequence that is
presented one bit at a time. The binary sequence is generated
by a program for the reference machine.

The agent must (as stated by the first principle) chose a
program (whose output must be consistent with anything that we
have seen in case we have made observations) for the reference
machine and then use its output (which can be of finite or
infinite length) as a prediction. If we want to predict at
least $h$ bits we have to restrict ourself to machines that
output at least $h$ bits. We will consider an enumeration of
all programs $T_i$. We are also going to consider a class of
reward structures $R_{i,j}$. The meaning is that if we guess
that the sequence is (as the output of) $T_i$ and the actual
sequence is $T_j$, then we receive reward $R_{i,j}$.  Note that
for any finite string there are always Turing machines that
computes it. We will furthermore suppose that $\forall i$,
$R_{i,j}\to 0$ as $j\to\infty$. This means that we consider it
to be a harder and harder task to guess $T_j$ as $j$ gets
really large. This assumption is not strictly necessary as we
will discuss later.

\subsection{Outline}
Section \ref{back} provides background on Solomonoff induction
and AIXI. In Section \ref{rat} we deal with the first two
principles mentioned above about rationality and computability.
In Section \ref{represent}, we discuss the third principle
which defines a prior from a (Universal Turing Machine)
representation. Section \ref{sp} describes the sequence
prediction algorithm that results from adding the fourth
principle to what has been achieved in the previous sections.
Section \ref{aixi} extends our analysis to the case where an
agent takes a sequence of actions that may affect its
environment. Section \ref{semi} concerns equivalence between
our beliefs over deterministic environments and beliefs over a
much larger class of stochastic environments.

\section{Background}\label{back}

\subsection{Sequence Prediction}\label{BSP}
We consider both finite and infinite sequences from a finite
alphabet $\mathcal{X}$. We denote the finite strings by
$\mathcal{X}^*$ and we use the notation
$x_{1:t}:=x_1,x_2,...,x_t$ for the first $t$ elements in a
sequence $x$. A function $\rho:\mathcal{X}^*\to[0,1]$ is a
probability measure if
\begin{equation}\label{measure}
  \rho(x)=\sum_{a\in \mathcal{X}} \rho(xa)\ \forall x\in\mathcal{X}^*
\end{equation}
and $\rho(\epsilon)=1$ where $\epsilon$ is the empty string.
Such a function describes a priori probabilistic beliefs about
the sequence. If the equality in \eqref{measure} is instead
$\geq$ and $\rho(\epsilon)\leq 1$ then we have a semi-measure.
We define the probability of seeing the string $a$ after seeing
$x$ as being $\rho(a|x):=\rho(xa)/\rho(x)$. If we have a loss
function $L:\mathcal{X}\times\mathcal{X}\to\mathbb{R}$, we
(\cite{Hutter07}) choose, after seeing the string $x$, to
predict
\begin{equation}
  \argmin_{a\in\mathcal{X}} \sum_{b\in\mathcal{X}}L(a,b)\rho(b|x).
\end{equation}
More generally, if we have an alphabet $\mathcal{Y}$ of actions
we can take and a loss function
$L:\mathcal{Y}\times\mathcal{X}\to\mathbb{R}$ we make the
choice
\begin{equation}
  \argmin_{a\in\mathcal{Y}} \sum_{b\in\mathcal{X}}L(a,b)\rho(b|x).
\end{equation}

\subsection{The Solomonoff Prior}
Ray Solomonoff \cite{Sol60} defined a set of priors that only
differ by a multiplicative constant. We call them Solomonoff
priors. To define them we need to first introduce some notions
about Turing machines \cite{Turing}.

A \emph{monotone Turing machine} $T$ (which we will just call
Turing machine and whose exact technical definition can be
found in \cite{LiV90}) is a function from a set of (binary)
strings to binary sequences that can either be finite or
infinite. We demand that it be possible to describe the
function as a machine with unidirectional input and output
tapes, read/write heads, a bidirectional work tape and a finite
state machine that decides the next action of the machine given
the symbols under the head on the input and work tape. The
input tape is read only and the output tape is write only. We
write that $T(p)=x*$ if output of $T$ starts with $x$ when
given input (\emph{program}) $p$.

A \emph{universal Turing machine} is a Turing machine that can
emulate all other Turing machines in the sense that for every
Turing machine $T$ there is at least one prefix $p$, such that
when $px$ is fed to the universal Turing machine, it computes
the same output as $T$ would when fed $x$ (See
\cite{LiV90,Hutter04} for further details).

A sequence is called \emph{computable} if some Turing machine
outputs it, or in other words, if for every universal Turing
machine there is a program $p$ that leads to this sequence
being the output.

We can also define what we will call a \emph{computable
environment} from a Turing machine. A computable environment is
something which you (an agent) feed an action to and the
environment outputs a string which we call a perception. We can
for example have a finite number of possible actions and we put
one after another on the input tape of the machine. We wait
until the previous input has been processed and one of finitely
many outputs has been produced. The machine might halt after a
finite number of actions have been processed or it might run
for ever.

\begin{definition}[Semi-measure from Turing machine]\label{semim}
Given a Turing machine $T$, we let
\begin{equation}\label{SM}
  \lambda_T(x):=\sum_{p:T(p)=x*} 2^{-l(p)}
\end{equation}
where $l(p)$ is the length of the program (input) $p$ and
$T(p)=x*$ means that $T$ starts with outputting $x$ when fed
$p$, though it might continue and output more afterwards.
\end{definition}

If the Turing machine $T$ in Definition \ref{semim} is
universal we call $\lambda_T$ a Solomonoff distribution.
Solomonoff induction is defined by letting $\rho$ in Section
\ref{BSP} be the Solomonoff prior for some universal Turing
machine. If $U$ is a universal Turing machine and $T$ is any
Turing machine there exists a constant $c>0$ (namely
$2^{-l(q)}$ where $q$ is the prefix that encodes $T$ in $U$)
such that
\begin{equation}\label{dom}
  \lambda_U(x)\geq c\lambda_T(x)\ \forall x\in\mathcal{X}^*.
\end{equation}
The set $\{\lambda_T\ |\ T\ \text{Turing}\}$ can be identified with
\cite{LiV90} with all lower semi-computable semi-measures (see
\cite{LiV90} for definitions and proofs). The property expressed by
\eqref{dom} is called universality (or dominance) and is the key to
proving the strong convergence results of Solomonoff Induction
\cite{Sol78,LiV90,Hutter04,Hutter07}.

\subsection{AIXI}
In the active case where an agent is taking a
sequence of actions to achieve some sort of objective, we are trying
to determine the best \emph{policy} $\pi$, defined as a function
from a history $a_1q_1,...,a_tq_t$ of actions $a_t$ and perceptions
$q_t$ to a choice of the next action $a_{t+1}$. The function $\rho$
from the sequence prediction case is in the active case of the form
$\rho(q_1,...,q_t|a_1,...,a_t)$ and represent the probability of
seing $q_1,...,q_t$ given that we have chosen actions $a_1,...,a_t$.
We can again define a ``learning" algorithm by conditioning on what
we have seen to define
\begin{equation}\label{cond}
  \rho(q_{t+1},...,q_{t+k}|q_1,...,q_t,a_1,...,a_{t+k})
  := \frac{\rho(q_1,...,q_{t+k}|a_1,...,a_{t+k})}{\rho(q_1,...,q_t|a_1,...,a_t)}.
\end{equation}
If $a_t=\pi(a_1q_1,...,a_{t-1}q_{t-1})\ \forall t$ and
$q=q_1,q_2,...$, then we also write $\rho(q|\pi)$ for the left
hand side in \eqref{cond}.

Suppose that we have an enumerated set of policies $\{\pi_i\}$
to choose from. Given a definition of reward $R(q)$ for a
sequence of percepts $q=q_1,q_2,...$ that can for example be
defined as in reinforcement learning by splitting $q_t$ into
observation $o_t$ and reward $r_t$ and using a discounted
reward sum $\sum_t \gamma^tr_t$ \cite{SutBar,Hutter04}, then we
can define
\begin{equation}
  R(\pi):=\mathbb{E}_\rho R(q):=\sum_q R(q)\rho(q|\pi)
\end{equation}
and make the choice
\begin{equation}
  \pi^*:=\argmax_\pi R(\pi).
\end{equation}
If we have a class of environments $\{T_j\}$ (say the
computable environments) and if $\rho$ is defined by saying
that we assign probability $p_j$ to $T_j$ being the true
environment, then we let $R_{i,j}=R(q)$ if $q$ is the sequence
of perceptions resulting from using policy $\pi_i$ in
environment $T_j$. Then $R(\pi_i)=\sum_j p_jR_{i,j}$ and we
choose the policy with index
\begin{equation}
  \argmax_i \sum p_jR_{i,j}.
\end{equation}
As outlined in \cite{Hutter04}, one can choose a Solomonoff
distribution also over active environments. The resulting agent
is referred to as AIXI.

\section{Choosing a Program}\label{rat}

In this section we describe the setup of the second principle
mentioned in the introduction, namely rationality. The section
is much briefer than what is suitable for the topic and we
refer the reader to our companion paper \cite{Sunehag11} for a
more comprehensive treatment. Rationality is meant in the sense
of internal consistency \cite{Sugden}, which is how it has been
used in \cite{NeuMor44} and \cite{Sav54}. We set up simple
axioms for a rational decision maker, which implies that the
decisions can be explained (or defined) from probabilistic
beliefs. The approach to probability by \cite{Ram31,deF37} is
interpreting probabilities as fair betting odds. There is an
intuitive similarity between our setup to the idea of
explaining/deriving probabilities as a bookmaker's betting odds
as done in \cite{deF37} and \cite{Ram31}.

Before we consider the question regarding which program we want
to choose we will first consider the question if we are
prepared to accept guessing $T_i$ for a given $R=\{R_{i,j}\}$
(i.e.\ accepting this bet). We suppose that the alternative is
to abstain (reject) and receive zero reward. We introduce
rationality axioms and prove that we must have probabilistic
beliefs over the possible sequences. Note that for any given
$i$, we have a sequence $R_{i,j}$ in $c_0$ (the space of real
valued sequences that converge to $0$). We will set up some
common sense rationality axioms for the way we make our
decisions. We will demand that a decision can be taken for any
reward structure $r$ ($R_{i,j}$ with fixed $i$) from $c_0$. If
$r$ is acceptable and $\lambda\geq 0$ then we want $\lambda r$
to be acceptable since this is simply a multiple of the same.
We also want the sum of two acceptable reward structures to be
acceptable. If we cannot lose (receive negative reward) we are
prepared to accept while if we are guaranteed to gain we are
not prepared to reject it. We cannot remove any axiom without
losing the conclusion.

\begin{definition}[Rationality]\label{rational}
Suppose that we have a function $z:c_0\to\{-1,1,0\}$ defining
the decision reject/accept/either $(-1/1/0)$ and $Z=\{r\in c_0\
|\ z(r)\in \{0,1\}\}$.
\begin{enumerate}
\item $z(r)\in \{0,1\}$ if and only if $z(-r)\in\{-1,0\}$
\item $r,s\in Z$, $\lambda,\gamma \geq 0$ then $\lambda r+\gamma s\in Z$
\item If $r_k\geq 0\ \forall k$ then $r\in Z$ while
if $r_k>0\ \forall k$ then $z(r)=1$.
\end{enumerate}
\end{definition}

The following theorem connects our Rationality axioms with the
Hahn-Banach theorem \cite{Krey89} and concludes that rational
decisions can be described with a positive continuous linear
functional on the space of reward structures. The Banach space
dual of $c_0$ is $\ell_1$ which gives us a probabilistic
representation of underlying beliefs.

\begin{theorem}[Linear separation]\label{thm:LF}
Given the assumptions in Definition \ref{rational} there exists
a positive continuous linear functional $f:c_0\to\mathbb{R}$
defined by $f(r)=\sum_j r_jp_j$ where $r=\{r_j\}$, $p_j\geq 0$
and $\sum_j p_j<\infty$, such that
\begin{equation}
  \{x\ |\ f(r)> 0\}\subseteq Z\subseteq\{r\ |\ f(r)\geq 0\}.
\end{equation}
\end{theorem}
\begin{proof}
The second property tells us that $Z$ and $-Z$ are convex
cones. The first and third property tells us that $Z\neq
\mathbb{R}^m$. Suppose that there is a point $r$ that lies in
both the interior of $Z$ and of $-Z$. Then the same is true for
$-r$ according to the first property and for the origin. That a
ball around the origin lies in $Z$  means that $Z=\mathbb{R}^m$
which is not true. Thus the interiors of $Z$ and $-Z$ are
disjoint open convex sets and can, therefore, be separated by a
hyperplane (according to the Hahn-Banach theorem) which goes
through the origin (since according to the first and third
property $z(0)=0$). The first property tell us that $Z\cup
-Z=\mathbb{R}^m$. Given a separating hyperplane (between the
interiors of $Z$ and $-Z$), $Z$ must contain everything on one
side. This means that $Z$ is a half space whose boundary is a
hyperplane that goes through the origin and the closure
$\bar{Z}$ of $Z$ is a closed half space and can be written as
$\{r\ |\ f(r)\geq 0\}$ for some $f$ in the Banach space dual
$c_0'=\ell_1$ of $c_0$. The third property tells us that $f$ is
positive.
\end{proof}

Theorem \ref{thm:LF} also leads us to how to choose between
different options. If we consider picking $T_{i}$ over $T_k$ we
will do (accept) that if $R_{i,\cdot}-R_{k,\cdot}$ is accepted.
This is the case if $\sum p_jR_{i,j}>\sum p_jR_{k,j}$. The
conclusion is that if we are presented with $R_{i,j}$ and a
class $\{T_j\}$ and we assign probability $p_j$ to $T_j$ being
the truth, then we choose

\begin{equation}
  \argmax_i \sum_j R_{i,j}p_j.
\end{equation}

\begin{remark}\label{linf}
If we replace the space $c_0$ by $\ell_\infty$ as the space of
reward structures in Theorem \ref{thm:LF}, the conclusion (see
\cite{Sunehag11}) is instead that $f$ is in the Banach space
dual $\ell_\infty'$ of $\ell_\infty$ which contains $\ell_1$
(the countably additive measures) but also functions that
cannot be written on the form $f(r)=\sum_j r_jp_j$.
$\ell_\infty'$ is sometimes called the ba space
\cite{Diestel84} and it consists of all finitely additive
measures.
\end{remark}

\section{Representation}\label{represent}

In this section we will discuss how indifference together with
a representation leads to a choice of prior weights. The
representation will be given in terms of codes that are strings
of letters from a finite alphabet and it tells us which
distinctions we will apply our indifference principle to.
Choosing the first bit can be viewed as choosing between two
propositions, e.g.\ $x$ is a vegetable or $x$ is a fruit. More
choices follow until a full specification (a code word for the
given reference machine) is reached. The section describes the
usual material on the Solomonoff distribution (see
\cite{LiV90}) in a way that highlights in what sense it is
based on indifference. The indifference principle itself is an
external behavioural principle.

\begin{definition}[Indifference]
Given a reward structure for two alternative outcomes of an
event where we receive $R_1$ or $R_2$ depending on the outcome,
then if we are indifferent we accept this bet if $R_1+R_2>0$.
For an agent with probabilistic beliefs that maximize expected
utility this means that equal probability is assigned to both
possibilities.
\end{definition}

We will discuss examples that are based on considering the set
$\{$apple, orange, carrot$\}$ and the representation that is
defined by first separating fruit from vegetables and then the
fruits into apples and oranges.

\begin{example}
We are about to open a box within which there is either a fruit
or a vegetable. We have no other information (except possibly,
a list of what is a fruit and what is a vegetable).
\end{example}

\begin{example}
We are about to open a box within which there is either an
apple, or an orange or a carrot. We have no other information.
\end{example}

Consider a representation where we use binary codes. If the
first digit is a $0$ it means a vegetable, i.e. a carrot. No
more digits are needed to describe the object. If the first
digit is a $1$ it means a fruit. If the next digit after the
$1$ is a $0$ its an apple and if it is a $1$ its an orange. In
the absence of any other background knowledge/information and
given that we are going to be indifferent for this choice, we
assign uniform probabilities for each choice of letter in the
string. For our examples this results in probabilities
$Pr($fruit$)=Pr($vegetable$)=1/2$. After concluding this we
consider the next distinction and conclude that
$Pr($apple$|$fruit$)=Pr($orange$|$fruit$)=1/2$. This means that
the decision maker has the prior beliefs $Pr($carrot$)=1/2$,
$Pr($apple$)=Pr($orange$)=1/4$.

An alternative representation would be to have a trinary
alphabet and give each object its own letter. The result of
this is $Pr($apple$)=Pr($orange$)=Pr($carrot$)=1/3$,
$Pr($fruit$)=2/3$ and $Pr($vegetable$)=1/3$.

The following formalizes the definition of a code and a prefix
free code. Since we are assuming that the possible outcomes are
never special cases of each other we need our code to be prefix
free. Furthermore, Kraft's inequality says that
$\sum_{c\in\mathcal{C}} 2^{-length(c)}\leq 1$ if the set of
codes $\mathcal{C}$ is prefix free.

\begin{definition}[Codes]
A code for a set $\mathcal{A}$ is a set of strings
$\mathcal{C}$ of letters from a finite alphabet $\mathcal{B}$
and a surjective map from $\mathcal{C}$ to $\mathcal{A}$. We
say that a code is prefix-free if no code string is a proper
prefix of another.
\end{definition}

\begin{definition}[Computable Representation]
We say that a code is a computable representation if the map
from code-strings to outcomes is a Turing machine.
\end{definition}

In the definition below we provide the formula for how a binary
representation of the letters in an alphabet leads to a choice
of a distribution. It is easily extended to non-binary
representations.

\begin{definition}[Distribution from representation]\label{prior}
Given a binary prefix-free code for $\mathcal{A}$ (our possible
outcomes), the expression
$$
  w_a=\sum_{\text{c code for a}}2^{-length(c)},\ a\in\mathcal{A}
$$
defines a measure over $\mathcal{A}$.
\end{definition}

Though the formula in Definition \ref{prior} uniquely
determines the weights given a representation, there is still a
very wide choice of representations. We are going to deal with
this concern to restrict ourself to the class of universal
representations with the property that given any other
computable representation, the universal weights are at least a
constant times the weights resulting from the other
representation. See \cite{Sol60,LiV90,Hutter04} for a more
extensive treatment. These universal representations are
defined by having a universal Turing machine (in our case the
given reference machine) as the map from codes to outcomes.

\begin{definition}[Universal Representation]
If a universal Turing machine is used for defining the map from
codes to outcomes we say that we have a universal (computable)
representation.
\end{definition}

The weights that result from using a universal representation
$w^U_a$ satisfy the property that if $w_a$ are the resulting
weights from another computable representation, then there is
$C>0$ such that $w^U_a\geq C w_a\ \forall a\in\mathcal{A}$.
This follows directly from the universality of the Turing
machine, which means that any other Turing machine can be
simulated on the universal one by adding an extra prefix
(interpreter) to each code. That is, feeding $ic$ to the
universal machine gives the same output as feeding $c$ to the
other machine. The constant $C$ is $2^{-length(i)}$.

\begin{theorem}
Applying Definition \ref{prior} together with a representation
of finite strings based on a universal Turing machine gives us
the Solomonoff semi-measure.
\end{theorem}
\begin{proof}
Given a universal Turing machine $U$ we create a set of codes
$\mathcal{C}$ from all programs that generate an output of at
least $h$ bits. We let the code $c\in\mathcal{C}$ represent the
finite string $x\in\mathcal{X}^*$ with $l(x)=h$ if $U(c)=x*$.
We show below that this representation together with Definition
\ref{prior} leads to the Solomonoff distribution for the next
$h$ bits. By considering all $h\geq 1$ we recover the
Solomonoff semi-measure over $\mathcal{X}^*$.

Formally, given $x\in\mathcal{X}^*$ we let (in Definition
\ref{prior}) $a=x$ and we define $\rho(x):=w_a$ and conclude
that
$$
  \rho(x)=\sum_{U(p)=x*}2^{-length(p)}
$$
which is the Solomonoff semi-measure.
\end{proof}

\begin{remark}[Unique Representation]\label{unique}
Given a universal Turing machine, we could choose to let only
the shortest program that generates a certain output represent
that output, and not all the programs that generate this
output. The length of the shortest program $p$ that gives
output $x$ is called the Kolmogorov complexity $K(x)$ of $x$.
Using only the shortest program leads to the slightly different
weights
$$
  w_x=2^{-K(x)}
$$
compared to Definition \ref{prior}. Both weighting schemes are,
however, equivalent within a multiplicative constant
\cite{LiV90}.
\end{remark}

\section{Sequence Prediction}\label{sp}

We will in this section summarize how Solomonoff Induction as
described in \cite{Hutter07} follows from what we have
presented in Section \ref{rat} and Section \ref{represent}
together with our fourth principle of time consistency.
Consider a binary sequence that is revealed to us one bit at a
time. We are trying to predict the future of the sequence,
either one bit, several bits or all of them. By combining the
conclusions of Section \ref{rat} and \ref{represent}, we can
define a sequence prediction algorithm which turns out to be
Solomonoff Induction. The results from Section \ref{rat} tells
us that if we are going to be able to make rational guesses
about which computable sequence we will see, we need to have
probabilistic beliefs.

If we are interested in predicting a finite number of bits we
need to design the reward structure in Section \ref{rat} to
reflect what we are interested in. If we want to predict the
next bit we can let $R_{i,j}=1$ if $T_i$ and $T_j$ have the
same next bit and $R_{i,j}=-1$ otherwise. This leads to (a
weighted majority decision to) predicting $1$ if $\sum_{j| T_j\
\text{produces}\ 1} p_j>\sum_{j| T_j\ \text{produces}\ 0}p_j$
and $0$ if the reverse inequality is true. The reasoning and
result generalizes naturally to predicting finitely many bits
and we can interpret this as minimizing the expected number of
errors.

\subsection{Updating}
Suppose that we have observed a number of bits of the
sequences. This result in contradictions with many of the
sequences and they can be ruled out. We next formally state the
fourth principle from the introduction.

\begin{definition}[Time-consistency]
Suppose that we are observing a sequence $x_1,x_2,...$ one bit
at a time ($x_t$ at time $t$). Suppose that we (at time $t$)
want to predict the next $h$ bits of a sequence and our
decisions (for any $t$ and $h$) are defined by a function
$z^t_h$ from the set of all reward structures
($\mathbb{R}^{m\times m}$ where $m=2^h$ in the binary case) to
the set of strings of length $h$.

Suppose that if $z^t_{h+1}(r)=y$ and $y$ starts with $x_{t+1}$.
If it then follows that $z_h^{t+1}(r')=y$ where $r'$ is the
restriction of $r$ to the strings that start with $x_{t+1}$
(and we identify such a string of length $h+1$ with the string
of length $h$ that follow the first bit) and if this
implication is true for any $t,r,h$ we say that we have
time-consistency.
\end{definition}

\begin{theorem}
Suppose that we have a semi-measure $\rho:\mathcal{X}^*\to
[0,1]$ and that we at time $0$ (given any loss $L$) predict the
next $h$ bits according to
\begin{equation}\label{genLoss}
  \argmin_{y_1\in\mathcal{X}^h}
  \sum_{y_2\in\mathcal{X}^h}L(y_1,y_2)\rho(y_2).
\end{equation}
If we furthermore assume time-consistency and observe
$x\in\mathcal{X}^*$, then we predict
\begin{equation}\label{condPred}
  \argmin_{y_1\in\mathcal{X}^h}
  \sum_{y_2\in\mathcal{X}^h}L(y_1,y_2)\rho(xy_2|x).
\end{equation}
\end{theorem}
\begin{proof}
Suppose that there are $y_1, y_2$ and $x$ such that
$\frac{\rho(xy_1|x)}{\rho(xy_2|x)}\neq
\frac{\rho(xy_1)}{\rho(xy_2)}$. This obviously contradicts
time-consistency. In other words, time-consistency implies that
relative beliefs in strings that are not yet contradicted
remains the same. Therefore, the decision function after seeing
$x$ can be described by a semi-measure where the inconsistent
alternatives have been ruled out and the others just
renormalized. This is what \eqref{condPred} is describing. The
only remaining point to make is that we have expressed
\eqref{genLoss} and \eqref{condPred} in terms of loss instead
of reward though it is simply a matter of changing the sign and
max for min.
\end{proof}

\section{The AIXI Agent}\label{aixi}

In this section we discuss extensions to the case where an
agent is choosing a sequence of actions that affect the
environment it is in. We will simply replace the principle that
says that we predict computable sequences by one that says that
we predict computable environments. The environments are such
that the agent takes an action that is fed to the environment
and the environment responds with an output that we call a
perception. There is a finite alphabet for the action and one
for the perception.

Our aim is to choose a policy for the agent. This is a function
from the history of the actions and perceptions that has
appeared so far, to the action which the agent chooses next.
Suppose that a class $\{\pi_i\}$ of policies, a class of (all)
computable environments $\{T_j\}$ and a reward structure
$R_{i,j}$ which is the total reward for using policy $\pi_i$ in
environment $T_j$. To assume the property that $\lim_j
R_{i,j}=0\ \forall i$, would mean that we assume that the
stakes are lower in the environments of high index. This
somewhat restrictive and there are alternatives to making this
assumption (that the reward structure is in $c_0$) and we
investigate the result of assuming that we instead have the
larger space $\ell_\infty$ (see Remark \ref{linf}) in a
separate article \cite{Sunehag11} on rationality axioms and
conclude that the difference is that we get finite additivity
instead of countable additivity for the probability measure but
that we can get back to countable additivity by adding an extra
monotonicity assumption. The arguments in Section \ref{rat}
imply (given $c_0$ reward structure) that we must assign
probabilities $\{p_j\}$ for the environment being $T_j$ and
choose a policy with index
\begin{equation}
  \argmax_i \sum_j R_{i,j}p_j.
\end{equation}
This is what the AIXI agent described in \cite{Hutter04} is
doing. The AIXI choice of weights $p_j$ correspond to the
choice $2^{-K(\nu)}$ (as in Remark \ref{unique}), but for the
class of lower semi-computable $\nu$ discussed below in Section
\ref{semi}.

The same updating technique as in Section \ref{sp}, where we
eliminate the environments which are inconsistent with what has
occurred, is being used. This is deduced from the same
time-consistency principle as for sequence prediction, just
stating that the relative belief in environments that are still
consistent will remain unchanged. This leads to the AIXI agent
from \cite{Hutter04}.

\section{Remarks on Stochastic Lower\\ Semi-computable Environments}\label{semi}

Having the belief that the environment is computable does seem
like a restrictive assumption though we will here argue that it
is in an interesting way equivalent to having beliefs over all
lower semi-computable stochastic environments. The Solomonoff
prior is based on having belief $2^{-l(p)}$ in having input
program $p$ defining the environment. We can (proven up to a
multiplicative factor in \cite{LiV90} and exact identity in
\cite{Wood11}), however, rewrite this prior as a mixture
$\sum_\nu w_\nu \nu$ over all lower semi-computable
environments $\nu$ where $w_\nu>0$ for all $\nu$. Therefore,
acting according to our Solomonoff mixture over computable
enviroments is identical to acting according to beliefs over a
much larger set of environments where we have randomness.

\section{Conclusions}\label{conclusion}

We defined four principles for universal sequence prediction
and showed that Solomonoff induction and AIXI are determined
from them. These principles are computability, rationality,
indifference and time consistency. Computability tells us that
Turing machines are the explanations we consider for what we
are seeing. Rationality tells us that we have probabilistic
beliefs over these. Time-consistency leads to the conclusion
that we update these beliefs based on conditional probability
and the principle of indifference tells us how to chose the
original beliefs based on how compactly the various Turing
machines can be implemented on the reference machine.

\paragraph{Acknowledgement.}
This work was supported by ARC grant DP0988049.

\addcontentsline{toc}{section}{\refname}

\begin{small}

\end{small}

\end{document}